%% file: main.tex
\documentclass{article}

\usepackage[main, preprint]{neurips_2026}

\usepackage[utf8]{inputenc}
\usepackage[T1]{fontenc}
\usepackage{hyperref}
\usepackage{url}
\usepackage{microtype}
\usepackage{graphicx}
\usepackage{subcaption}
\usepackage{booktabs}
\usepackage{multirow}
\usepackage{xcolor}
\usepackage{nicefrac}
\usepackage{amsmath}
\usepackage{amssymb}
\usepackage{amsfonts}
\usepackage{mathtools}
\usepackage{amsthm}
\usepackage{algorithm}
\usepackage{algorithmic}
\usepackage{float}
\usepackage{comment}
\usepackage{tikz}
\usepackage{pgfplots}
\pgfplotsset{compat=1.18}

\providecommand{\State}{\STATE}
\providecommand{\For}{\FOR}
\providecommand{\EndFor}{\ENDFOR}

\input{math_commands.tex}

\input{lpsymbols}

\theoremstyle{plain}
\newtheorem{theorem}{Theorem}[section]

\theoremstyle{definition}

\theoremstyle{remark}

\newcommand{\chisq}{\chi^2}

\makeatletter
\renewcommand{\eqref}[1]{\textup{\tagform@{\ref{#1}}}}
\makeatother

\title{Distributionally Robust Token Optimization in RLHF}

\author{%
  Yeping Jin \\
  Department of System Engineering \\
  Boston University, Boston, MA 02215, USA \\
  \texttt{yepjin@bu.edu} \\
  \And
  Jiaming Hu \\
  Department of Math \& Statistics \\
  Boston University, Boston, MA 02215, USA \\
  \texttt{jh7453@bu.edu} \\
  \And
  Ioannis Ch. Paschalidis \\
  Department of System Engineering \\
  Boston University, Boston, MA 02215, USA \\
  \texttt{yannisp@bu.edu} \\
}

\begin{document}

\maketitle

\begin{abstract}
Large Language Models (LLMs) tend to respond correctly to prompts that align well with the data they were trained and fine-tuned on. Yet, small shifts in wording, format, or language can trigger surprisingly large failures, especially on multi-step reasoning problems. To address this problem, we propose a \textbf{Distributionally Robust Token Optimization (DRTO)} approach, which combines token-level Reinforcement Learning from Human Feedback (RLHF) with  \textit{Distributionally Robust Optimization (DRO)}. DRTO constructs f-divergence ambiguity sets over span-level actor losses, providing a principled way to emphasize difficult response segments during policy optimization. Empirically, DRTO enhances consistency under distribution shifts in multiple reasoning benchmarks among different tasks, achieving $+4.4$ percentage points on \texttt{MATH-500} and $+2.7$ percentage points on \texttt{LiveCodeBench} over standard RTO. Our code is available at \href{https://osf.io/zwft2/overview?view_only=857bc10a58f34a6e9a39bd2d761d1fa5}{OSF} .
\end{abstract}

\section{Introduction}

Large language models (LLMs) have achieved strong performance on a wide range of language and reasoning tasks, and reinforcement learning from human feedback (RLHF) has become a practical and widely used approach to align these models with human preferences. In modern post-training pipelines, two mainstream choices are Proximal Policy Optimization (PPO)~\citep{ppo} and Direct Preference Optimization (DPO)~\citep{rafailov2023direct}. Despite their success, both methods can be vulnerable to distributional shifts between the preference data used for alignment and the prompts or evaluation conditions encountered at deployment. In practice, such vulnerability may appear as training instability and reduced robustness. Small changes in phrasing, notation, or other superficial input variations can lead to substantial drops in accuracy.

Due to this instability issue, recent work has begun to modify PPO and DPO models to enhance robustness~\citep{drdpo,chipo}. However, existing robust variants are often limited by increased conservatism or additional computational overhead, and they usually do not fully exploit the token-level credit assignment. As a result, they may struggle to improve out-of-distribution robustness without sacrificing strong in-distribution performance, especially on reasoning-intensive tasks.

To address these limitations, we develop a distributionally robust method, building on a recent token-level alignment framework, \emph{Reinforced Token Optimization} (RTO)~\citep{rto}, which already demonstrates consistent advantages over PPO and DPO by combining PPO-style policy optimization with a token-wise reward signal learned from preferences. We propose \emph{Distributionally Robust Token Optimization} (DRTO), which adapts \emph{Distributionally Robust Optimization} (DRO) to the RTO actor update through span-level divergence-based ambiguity sets, including KL and $\chi^2$ divergence. An overview of the resulting pipeline is shown in Figure~\ref{fig:overview}. Such a formulation yields more stable updates and improves consistency on out-of-distribution evaluations, while retaining RTO's performance advantages over DPO and PPO in our setting.

\begin{figure}[!htbp]
\centering
\includegraphics[width=\linewidth]{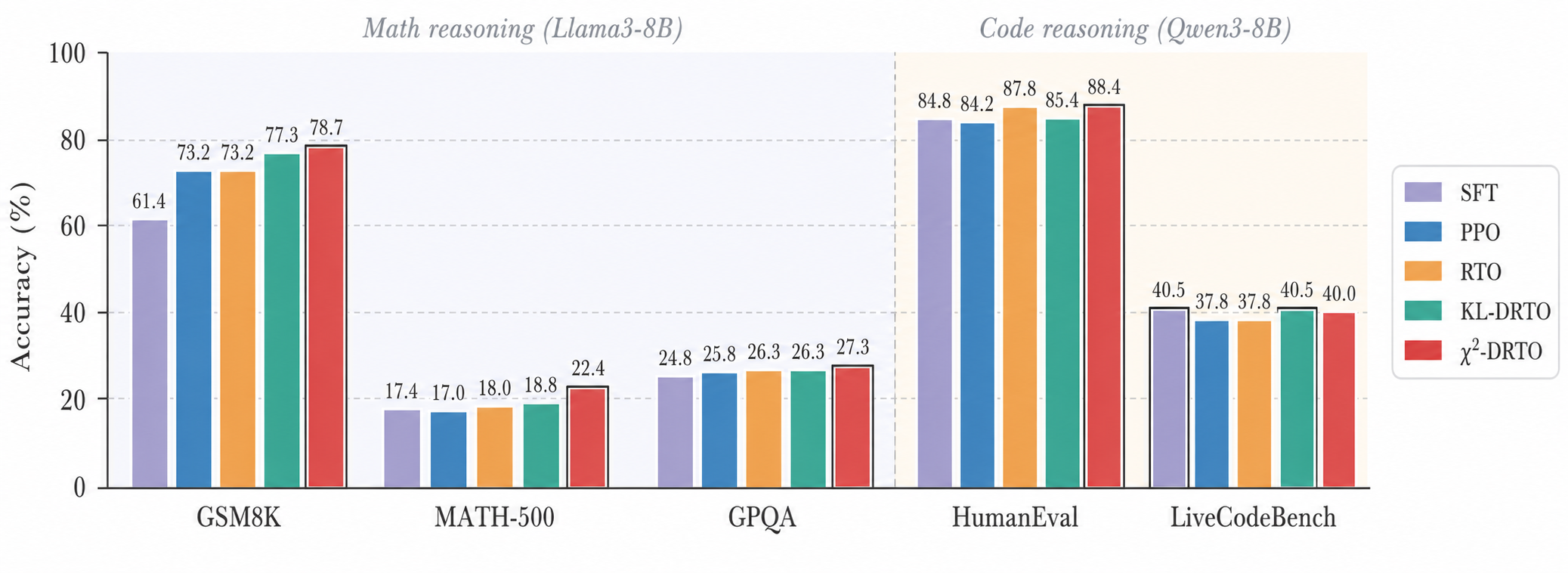}
\caption{Performance of five RLHF fine-tuning methods on math and coding benchmarks under distribution shift. DRTO improves over standard RTO on most benchmarks.}
\label{fig:headline}
\end{figure}

\subsection{Contributions}

This work enhances the robustness of RLHF via a span-level DRO approach. We summarize our main contributions as follows.
\begin{itemize}

\item \textbf{Distributionally robust span-level objectives.}
We introduce ambiguity sets over response spans and derive robust objectives for the actor loss. Each response is partitioned into contiguous spans, where each span is assigned a mean PPO surrogate loss, and DRO is applied directly over the resulting span losses. To the best of our knowledge, we are not aware of prior work that applies DRO over response segmentations within an RLHF framework.

\item \textbf{Realization of span DRO.} We use two distinct f-divergences to measure the ambiguity set, leading to two different practical realizations. For KL-DRTO, we obtain an entropic objective that smoothly emphasizes high-loss spans through exponential tilting; for $\chisq$-DRTO, we obtain a tractable surrogate based on the response-balanced span mean and standard deviation.

\item \textbf{Empirical improvements under distribution shifts.}
Our practical implementations of KL-DRTO and $\chisq$-DRTO use the same training pipeline as standard RTO, with little to no additional runtime or compute cost. Empirically, both methods yield more consistent performance under linguistic and symbolic shifts on math reasoning tasks, while preserving strong in-domain performance. As summarized in Figure~\ref{fig:headline}, across the five math benchmarks DRTO attains up to $+5.5$ percentage points absolute accuracy improvement over standard RTO, including $+4.4$ on \texttt{MATH-500}, and is the only family that never regresses on any math benchmark; on coding, $\chisq$-DRTO matches or exceeds RTO on every benchmark (including $+2.12$ percentage points on \texttt{LiveCodeBench}), while KL-DRTO obtains the largest \texttt{LiveCodeBench} gain of $+2.70$ percentage points. Subsequent ablation studies further validate DRTO's efficiency.
\end{itemize}

In summary, we develop a theoretically grounded and practically effective approach for robust token-level RLHF, which also inherits the reasoning advantages of token-level methods over standard PPO and DPO.

\begin{figure}[!htbp]
\centering
\includegraphics[width=\linewidth]{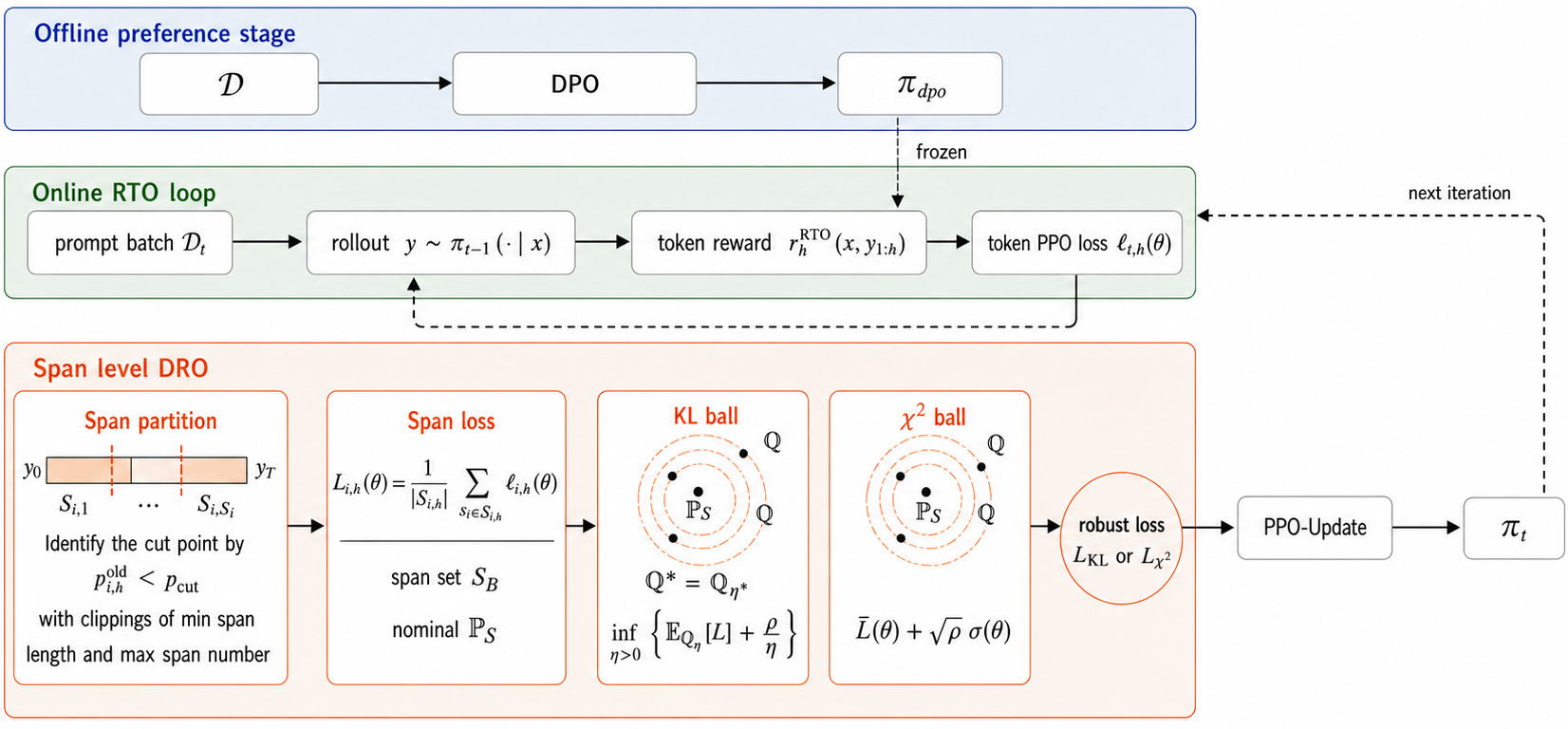}
\caption{Overview of the DRTO pipeline. Each rollout is partitioned into spans at low-confidence cut points (where $p^{\mathrm{old}}_{i,h}<p_{\mathrm{cut}}$); per-token PPO losses driven by the RTO token reward are aggregated into span losses, which are then reweighted by a KL- or $\chi^2$-divergence DRO adversary to form the final actor loss.}
\label{fig:overview}
\end{figure}

\subsection{Related Work}

\textbf{Token-level RLHF.} RLHF~\citep{ouyang2022training,bai2022constitutional} is widely used to align LLMs with human preferences. Among all popular frameworks, PPO~\citep{ppo} remains a common choice for reinforcement-based fine-tuning, typically with a KL control term to limit deviation from a reference policy. As an alternative mainstream approach, DPO~\citep{rafailov2023direct} reformulates preference optimization as a supervised objective on pairwise comparisons, avoiding explicit on-policy rollouts. To enhance the two popular methods with sequence-level objectives, token-level alignment methods aim to improve credit assignment and make training less brittle. For instance, Token-level DPO (TDPO)~\citep{zeng2024tdpo} introduces a token-wise formulation of the preference objective. In this paper, we build on Reinforced Token Optimization (RTO)~\citep{rto}, which adds a token-level reward that involves DPO model in the standard PPO framework, enabling more localized learning signals. Even so, token-level training can be sensitive to reward noise, difficult response segments, and minibatch composition, which motivates robust objectives that enhance stability under such uncertainty.

\textbf{Distributionally robust optimization for preference alignment.} DRO provides a principled framework for learning under distributional uncertainty~\citep{chen2021drl,rahimian2019distributionally,duchi2021statistics}, and divergence-based ambiguity sets yield tractable dual forms linked to variance-sensitive and entropic-risk objectives~\citep{namkoong2017variance}. Recent alignment work has applied DRO to offline preference optimization: DrDPO~\citep{drdpo} robustifies DPO against data noise via distributionally robust reweighting, and $\chi^2$-Preference Optimization~\citep{chipo} addresses offline over-optimization via $\chi^2$-based regularization. Group-robust preference optimization (GRPO)~\citep{grpo} maximizes worst-group performance when group labels are available. In contrast, DRTO enhances the robustness of the on-policy RTO actor loss by applying either $\chi^2$ ambiguity sets or KL regularization over span-level PPO losses, while also inheriting RTO's advantages over DPO and PPO under in-distribution settings.

\textbf{Segment-level RL for LLMs.} A growing line of work segments responses for finer-grained learning signals: \citet{yin2025segmenting} and \citet{li2024adaptive} learn segment-level rewards via dynamic or semantic boundaries. Our span construction shares the segment-level granularity but is used for divergence-based DRO reweighting rather than reward shaping or advantage smoothing.

\textbf{Robustness and safety beyond DRO.} A complementary line of work targets adversarial robustness and safety. White-box methods such as GradSafe~\citep{xie2024gradsafe}, Gradient Cuff~\citep{hu2024gradient}, and Robust Prompt Optimization~\citep{zhou2024robust} exploit internal model signals or gradient-based suffix optimization, while black-box defenses such as SmoothLLM~\citep{robey2023smoothllm} and CCFC~\citep{hu2025ccfc} aggregate over input perturbations or dual-track prompts. These approaches are orthogonal to DRTO: they target safety against adversarial prompts, whereas DRTO targets consistent performance under natural distribution shift.

\section{Preliminaries}
\label{sec:prelim}

In this section, we introduce the standard token-level optimization framework for RLHF and establish the notation used throughout the paper.

\subsection{Proximal Policy Optimization}

The RLHF objective maximizes a learned reward $r(x,y)$ while keeping the policy $\pi_\theta$ close to a reference $\pi_{\mathrm{ref}}$ via a KL penalty~\citep{ouyang2022training}. PPO optimizes this objective with clipped on-policy updates. Concretely, we sample a minibatch of $n$ rollouts $\mathcal B=\{\tau_i=(x_i,y_i)\}_{i=1}^n$, where each $x_i$ is an input prompt and $y_i=(y_{i,1},\ldots,y_{i,H_i})$ is the response of length $H_i$ generated by the rollout policy $\pi_{\mathrm{old}}$ (the snapshot of $\pi_\theta$ used to collect the current minibatch, refreshed at each outer iteration), with each token $y_{i,h}$ drawn from the model's vocabulary. At answer position $h\in[H_i]$, let $s_{i,h}:=(x_i,y_{i,1:h-1})$ denote the partial context (prompt and tokens generated so far) and $a_{i,h}:=y_{i,h}$ the next-token action; the PPO importance-sampling ratio is then $\rho_{i,h}(\theta)=\pi_\theta(a_{i,h}\mid s_{i,h})/\pi_{\mathrm{old}}(a_{i,h}\mid s_{i,h})$, where $\pi_{\mathrm{old}}$ is the rollout policy and is distinct from the frozen reference $\pi_{\mathrm{ref}}$ used only in the KL regularization term. The per-token advantage $\widehat A_{i,h}\in\mathbb R$ estimates how much better action $a_{i,h}$ is than the policy's average behavior in state $s_{i,h}$, computed from the per-token rewards using generalized advantage estimation (GAE)~\citep{schulman2015gae}; the clip function $\mathrm{clip}(u,1-\epsilon,1+\epsilon)$ truncates $u$ to the interval $[1-\epsilon,1+\epsilon]$. The PPO clipped surrogate then defines the token-level actor loss
\begin{equation}
\begin{aligned}
\ell_{i,h}(\theta)
:=-
\min\!\Big(
\rho_{i,h}(\theta)\widehat A_{i,h},
\mathrm{clip}\big(\rho_{i,h}(\theta),1-\epsilon,1+\epsilon\big)\widehat A_{i,h}
\Big),
\end{aligned}
\label{eq:ppo_token_loss_prelim}
\end{equation}
where $\epsilon>0$ is the clipping parameter. The response-level and minibatch PPO losses are
\begin{equation}
L_i^{\mathrm{resp}}(\theta):=\tfrac{1}{H_i}\sum_{h\in[H_i]}\ell_{i,h}(\theta),\qquad L_{\mathrm{PPO}}(\theta)=\tfrac{1}{n}\sum_{i=1}^n L_i^{\mathrm{resp}}(\theta).
\label{eq:ppo_surrogate_prelim}
\end{equation}
We use \texttt{PPO-Update} to denote one minimization step on this loss.

\subsection{Reinforced Token Optimization}
\label{subsec:prelim_rto}

RTO~\citep{rto} keeps the PPO actor update but replaces the sparse terminal reward with a dense token-wise shaping reward derived from a preference policy $\pi_{\mathrm{dpo}}$ trained from offline preference data. Starting from $\pi_{\mathrm{ref}}$, at answer-token position $h\in[H_i]$, the shaping reward is
\begin{equation}
\begin{aligned}
&r^{\mathrm{RTO}}_h(x,y_{1:h})
=
\beta_1 \log\frac{\pi_{\mathrm{dpo}}(y_h\mid x,y_{1:h-1})}
{\pi_{\mathrm{ref}}(y_h\mid x,y_{1:h-1})} \\
&\quad
-\beta_2 \log\frac{\pi_{\theta}(y_h\mid x,y_{1:h-1})}
{\pi_{\mathrm{ref}}(y_h\mid x,y_{1:h-1})}
+\beta_3 \mathbf{1}_{\{h=H_i\}}\,r(x,y_{1:H_i}),
\end{aligned}
\label{eq:rto_token_reward_prelim}
\end{equation}
where $\beta_1,\beta_2,\beta_3>0$ are tuning parameters and the last term is a terminal reward applied only at $h=H_i$. Algorithm~\ref{alg:rto_practical} formulates the procedure: at each iteration, the DPO model is used to compute token-wise rewards via \eqref{eq:rto_token_reward_prelim}, and a PPO update follows.

\begin{algorithm}[tb]
\caption{Standard RTO}
\label{alg:rto_practical}
\begin{algorithmic}[1]
\State \textbf{Input:} offline dataset $\mathcal D$; $\beta_1,\beta_2,\beta_3>0$; DPO algorithm \texttt{DPO}; PPO trainer \texttt{PPO-Update}.
\State Train $\pi_{\mathrm{dpo}}\leftarrow \texttt{DPO}(\mathcal D)$; initialize $\pi_{0}=\pi_{\mathrm{ref}}$.
\For{$t=1,\ldots,T$}
  \State Sample a batch of prompts $D_t$.
  \State For each $x\in D_t$, generate response $y\sim \pi_{t-1}(\cdot\mid x)$.
  \State Compute token-wise rewards $\{r^{\mathrm{RTO}}_h(x,y_{1:h})\}_{h\in[H]}$ by \eqref{eq:rto_token_reward_prelim}.
  \State Update $\pi_t \leftarrow \texttt{PPO-Update}(\pi_{t-1}, r^{\mathrm{RTO}}, D_t)$.
\EndFor
\State \textbf{Output:} $\pi_T$.
\end{algorithmic}
\end{algorithm}
\section{Span-Level Distributionally Robust Token Optimization}
\label{sec:method}

We now formulate a span-level DRO objective that models distribution shift via adversarial reweighting over response segments. Span-level reweighting can be viewed as a fine-grained generalization of trajectory-level reweighting: it can up-weight difficult segments inside each response while still aggregating to a valid actor loss. Throughout this section we work with the same on-policy minibatch $\mathcal B=\{\tau_i=(x_i,y_i)\}_{i=1}^n$ of $n$ prompt-response rollouts as in Section~\ref{sec:prelim}, sampled at each PPO-style iteration; the adversarial reweighting is applied within $\mathcal B$ at every update.

\subsection{Span-Level DRO Objective}
\label{subsec:span_dro_objective}

\paragraph{Span partition.}
For the $i$-th response, we partition the answer-token positions $[H_i]$ into $S_i$ contiguous spans
\[
\mathcal G_{i,1},\ldots,\mathcal G_{i,S_i},\qquad \mathcal G_{i,s}\subseteq[H_i].
\]
We place span cut points at low-confidence rollout positions, i.e., positions whose stored probability $p^{\mathrm{old}}_{i,h}$ falls below a cut probability threshold $p_{\mathrm{cut}}$, subject to a minimum span length and a maximum number of spans per response. Low-probability tokens often mark transitions between reasoning steps or syntactic units, so they form natural span boundaries; the full procedure is deferred to Appendix~\ref{app:implementation}. We let
\[
\mathcal S_{\mathcal B}:=\{(i,s):i\in[n],\ s\in[S_i]\}
\]
index all spans in the minibatch $\mathcal B$.

\paragraph{Span loss.}
Reusing the per-token PPO loss $\ell_{i,h}(\theta)$ in \eqref{eq:ppo_token_loss_prelim}, the loss of span $\mathcal G_{i,s}$ is its mean token loss
\begin{equation}
L_{i,s}(\theta):=\frac{1}{|\mathcal G_{i,s}|}\sum_{h\in\mathcal G_{i,s}}\ell_{i,h}(\theta),
\label{eq:span_loss_def}
\end{equation}
which recovers the response-level loss \eqref{eq:ppo_surrogate_prelim} when $S_i=1$.

\paragraph{Span-level DRO.}
We assign each response equal mass $1/n$ and split it uniformly across its spans, defining the response-balanced nominal span distribution
\begin{equation}
\mathbb{P}_{\mathcal S}(i,s):=\frac{1}{n S_i},\qquad (i,s)\in\mathcal S_{\mathcal B}.
\label{eq:span_nominal_distribution}
\end{equation}
Given an ambiguity set $\Omega(\mathbb{P}_{\mathcal S})$ around $\mathbb{P}_{\mathcal S}$, DRTO replaces the mean PPO loss with the robust value
\begin{equation}
R(\theta;\rho):=\sup_{\mathbb{Q}\in\Omega(\mathbb{P}_{\mathcal S})}\E_{(i,s)\sim\mathbb{Q}}\big[L_{i,s}(\theta)\big],
\label{eq:generic_span_dro}
\end{equation}
which reduces to $L_{\mathrm{PPO}}(\theta)$ when $\rho=0$ and $S_i\equiv 1$. The following sections instantiate \eqref{eq:generic_span_dro} using KL and Pearson $\chi^2$ divergences.

\subsection{DRTO on a KL-Divergence Ambiguity Set}
\label{subsec:kl_drto}

We first instantiate \eqref{eq:generic_span_dro} with a KL-divergence ambiguity set
\begin{equation}
\Omega^{\mathrm{KL}}_\rho(\mathbb{P}_{\mathcal S})
=
\left\{\mathbb{Q}:\ D_{\mathrm{KL}}(\mathbb{Q}\Vert \mathbb{P}_{\mathcal S})\le \rho\right\}.
\label{eq:kl_ball}
\end{equation}
On the finite span set, the inner maximization in \eqref{eq:generic_span_dro} admits an exact one-dimensional dual~\cite{HuHong2012KLDRO}. For any $\eta>0$, we define the entropic risk
\begin{equation}
\Psi_\eta(\theta)
:=
\frac{1}{\eta}\log\E_{(i,s)\sim\mathbb{P}_{\mathcal S}}\!\left[\exp\!\left(\eta L_{i,s}(\theta)\right)\right],
\label{eq:psi_def}
\end{equation}
and the corresponding exponential-tilt weights
\begin{equation}
\mathbb{Q}_\eta(i,s)
:=
\frac{\mathbb{P}_{\mathcal S}(i,s)\exp(\eta L_{i,s}(\theta))}
{\sum_{(j,r)\in\mathcal S_{\mathcal B}}\mathbb{P}_{\mathcal S}(j,r)\exp(\eta L_{j,r}(\theta))}.
\label{eq:softmax_Q_main}
\end{equation}

\begin{figure}[tb]
\begin{minipage}[t]{0.47\linewidth}
\begin{algorithm}[H]
\caption{KL-DRTO}
\label{alg:kl_drto_true}
\begin{algorithmic}[1]
\State \textbf{Input:} $\mathcal D$; $\beta_{1:3},\rho>0$; \texttt{DPO}; \texttt{PPO-Update}.
\State Train $\pi_{\mathrm{dpo}}\!\leftarrow\!\texttt{DPO}(\mathcal D)$; init $\pi_{0}\!=\!\pi_{\mathrm{ref}}$.
\For{$t=1,\ldots,T$}
    \State Sample $D_t$ and generate $y\sim\pi_{t-1}(\cdot|x)$.
    \State Compute $\{r^{\mathrm{RTO}}_h\}$ by \eqref{eq:rto_token_reward_prelim}.
    \State Build spans $\mathcal S_{\mathcal B}$ and $\{L_{i,s}\}$ via \eqref{eq:span_loss_def}.
    \State $\eta^\star\!=\!\arg\min_{\eta>0}\{\Psi_\eta(\theta)+\rho/\eta\}$ \eqref{eq:psi_def}.
    \State $\mathbb{Q}^\star\!=\!\mathbb{Q}_{\eta^\star}$ via \eqref{eq:softmax_Q_main}.
    \State Set $L_{\mathrm{KL}}(\theta)=\sum_{(i,s)}\mathbb{Q}^\star(i,s)L_{i,s}(\theta)$.
    \State $\pi_t\!\leftarrow\!\texttt{PPO-Update}(\pi_{t-1},r^{\mathrm{RTO}},L_{\mathrm{KL}})$.
\EndFor
\State \textbf{Output:} $\pi_T$.
\end{algorithmic}
\end{algorithm}
\end{minipage}\hfill
\begin{minipage}[t]{0.47\linewidth}
\begin{algorithm}[H]
\caption{$\chi^2$-DRTO}
\label{alg:chi2_drto}
\begin{algorithmic}[1]
\State \textbf{Input:} $\mathcal D$; $\beta_{1:3}>0$, $\rho,\delta>0$; \texttt{DPO}; \texttt{PPO-Update}.
\State Train $\pi_{\mathrm{dpo}}\!\leftarrow\!\texttt{DPO}(\mathcal D)$; init $\pi_{0}\!=\!\pi_{\mathrm{ref}}$.
\For{$t=1,\ldots,T$}
    \State Sample $D_t$ and generate $y\sim\pi_{t-1}(\cdot|x)$.
    \State Compute $\{r^{\mathrm{RTO}}_h\}$ by \eqref{eq:rto_token_reward_prelim}.
    \State Build spans $\mathcal S_{\mathcal B}$ and $\{L_{i,s}\}$ via \eqref{eq:span_loss_def}.
    \State $\bar L\!=\!\E_{\mathbb{P}_{\mathcal S}}[L_{i,s}]$, $\Var(L)\!=\!\E_{\mathbb{P}_{\mathcal S}}[(L_{i,s}\!-\!\bar L)^2]$.
    \State Set $L_{\chi^2}(\theta)=\bar L+\sqrt{\rho}\sqrt{\Var(L)+\delta}$.
    \State $\pi_t\!\leftarrow\!\texttt{PPO-Update}(\pi_{t-1},r^{\mathrm{RTO}},L_{\chi^2})$.
\EndFor
\State \textbf{Output:} $\pi_T$.
\end{algorithmic}
\end{algorithm}
\end{minipage}
\end{figure}

\begin{theorem}[Guarantee of KL-DRTO]
\label{thm:kl_drto_main}
Fix any minibatch $\mathcal B$ with span set $\mathcal S_{\mathcal B}$, and assume $L_{i,s}(\theta)$ is finite for all $(i,s)\in\mathcal S_{\mathcal B}$. Then the KL-robust value \eqref{eq:generic_span_dro} admits the exact dual representation
\begin{equation}
R_{\mathrm{KL}}(\theta;\rho)
=
\inf_{\eta>0}\left\{\Psi_\eta(\theta)+\frac{\rho}{\eta}\right\}.
\label{eq:kl_exact_dual_main}
\end{equation}
Moreover, if the infimum in \eqref{eq:kl_exact_dual_main} is attained at some $\eta^\star\in(0,\infty)$, then the worst-case reweighting is $\mathbb{Q}^\star=\mathbb{Q}_{\eta^\star}$, and
\begin{equation}
R_{\mathrm{KL}}(\theta;\rho)
=
\sum_{(i,s)\in\mathcal S_{\mathcal B}}\mathbb{Q}^\star(i,s)\,L_{i,s}(\theta).
\label{eq:kl_primal_value_main}
\end{equation}
\end{theorem}
\begin{proof}
See Appendix~\ref{app:kl_drto_proofs}.
\end{proof}

Theorem~\ref{thm:kl_drto_main} reduces the KL-ball adversary in \eqref{eq:generic_span_dro} to a scalar minimization over $\eta>0$ and yields an explicit worst-case reweighting over spans. Algorithm~\ref{alg:kl_drto_true} therefore implements KL-DRTO by computing $\eta^\star$ and $\mathbb{Q}^\star$ from the minibatch span losses $\{L_{i,s}(\theta)\}$ and running RTO with the weighted loss $L_{\mathrm{KL}}(\theta)=\sum_{(i,s)}\mathbb{Q}^\star(i,s)\,L_{i,s}(\theta)$.

\subsection{DRTO on a $\chi^2$-Divergence Ambiguity Set}
\label{subsec:chi2_drto}

We now instantiate \eqref{eq:generic_span_dro} with the Pearson $\chi^2$-divergence~\cite{namkoong2017variance}
\begin{equation}
\Omega^{\chi^2}_\rho(\mathbb{P}_{\mathcal S})
=
\left\{\mathbb{Q}:\ \E_{(i,s)\sim\mathbb{P}_{\mathcal S}}\!\left[\left(\frac{\mathbb{Q}(i,s)}{\mathbb{P}_{\mathcal S}(i,s)}-1\right)^2\right]\le \rho\right\},
\label{eq:chi2_ball}
\end{equation}
and the corresponding $\chi^2$ span-robust value is
\begin{equation}
R_{\chi^2}(\theta;\rho)
:=
\sup_{\mathbb{Q}\in\Omega^{\chi^2}_\rho(\mathbb{P}_{\mathcal S})}\E_{(i,s)\sim\mathbb{Q}}\!\left[L_{i,s}(\theta)\right].
\label{eq:chi2_robust_value}
\end{equation}

A tractable relaxation depends only on the span mean and standard deviation under $\mathbb{P}_{\mathcal S}$,
\[
\bar L(\theta):=\E_{(i,s)\sim\mathbb{P}_{\mathcal S}}[L_{i,s}(\theta)],
\qquad
\sigma(\theta):=\sqrt{\Var_{(i,s)\sim\mathbb{P}_{\mathcal S}}(L_{i,s}(\theta))}.
\]
We thus have the following result.

\begin{theorem}[Relaxation for $\chi^2$-DRTO]
\label{thm:chi2_drto_main}
Fix any minibatch $\mathcal B$ with span set $\mathcal S_{\mathcal B}$ and nominal distribution $\mathbb{P}_{\mathcal S}$, and assume $L_{i,s}(\theta)\in\mathbb{R}$ for all $(i,s)\in\mathcal S_{\mathcal B}$. Then for any $\theta$,
\begin{equation}
R_{\chi^2}(\theta;\rho)\;\le\;\bar L(\theta)+\sigma(\theta)\sqrt{\rho}.
\label{eq:chi2_bound_main}
\end{equation}
Moreover, if $\sigma(\theta)>0$ and the nonnegativity condition
\begin{equation}
1+\sqrt{\rho}\,\frac{L_{i,s}(\theta)-\bar L(\theta)}{\sigma(\theta)}\ge 0,\qquad \forall (i,s)\in\mathcal S_{\mathcal B},
\label{eq:chi2_tight_cond_main}
\end{equation}
holds, then the bound \eqref{eq:chi2_bound_main} is tight.
\end{theorem}

\begin{proof}
    See Appendix~\ref{app:chi2_drto_proofs}.
\end{proof}

The relaxation in \eqref{eq:chi2_bound_main} yields the stabilized span-level surrogate
\begin{equation}
L_{\chi^2}(\theta)
:=
\bar L(\theta)+\sqrt{\rho}\,\sqrt{\Var_{\mathbb{P}_{\mathcal S}}(L_{i,s}(\theta))+\delta},
\label{eq:chi2_drto_obj}
\end{equation}
where $\delta>0$ is a small numerical stabilizer. As $\delta\to 0$, \eqref{eq:chi2_drto_obj} matches the bound in \eqref{eq:chi2_bound_main}, which is tight whenever \eqref{eq:chi2_tight_cond_main} holds. Algorithm~\ref{alg:chi2_drto} implements $\chi^2$-DRTO by computing \eqref{eq:chi2_drto_obj} from the span losses' sample mean and variance under $\mathbb{P}_{\mathcal S}$.


\section{Experimental Results}
\label{subsec:experiments}
In this section, we empirically evaluate our DRTO methods to assess their capability of improving robustness under distribution shift and maintaining strong performance on in-distribution data. 

\subsection{Robustness of DRTO}
\label{subsec:main_exp}
We evaluate KL-DRTO and $\chi^2$-DRTO against several widely used RLHF fine-tuning objectives on two complementary reasoning domains: mathematical reasoning with \texttt{Llama3-8B} and code reasoning with \texttt{Qwen3-8B}. For both pipelines, we use an ambiguity radius of $0.1$ and $0.01$ for $\chi^2$-DRTO and KL-DRTO respectively, with minimum span length 16, maximum span number 8, and cut probability threshold $p_{\mathrm{cut}}=0.05$. Additional implementation details are provided in Appendix~\ref{app:implementation}.

\paragraph{Math reasoning setup.}
All math methods share the \texttt{Llama3-8B} backbone and are fine-tuned on 10{,}000 samples from \texttt{OpenMathInstruct}~\citep{openmath}, a dataset spanning diverse math topics and solution formats. We report 5-shot flexible-match accuracy via \texttt{lm-evaluation-harness}~\citep{lm-eval}. The five math benchmarks jointly probe in-distribution generalization and a range of distribution shifts: \texttt{GSM8K} (grade-school arithmetic, nearest in-distribution), \texttt{GSM(CoT)} (same problems with a chain-of-thought trigger, prompt-format shift), \texttt{GSM(ES)} (Spanish translation, language shift), \texttt{MATH-500} (harder competition-style problems), and \texttt{GPQA} (graduate-level science, difficulty and topic shift).

\paragraph{Code reasoning setup.}
For code reasoning, all methods share the \texttt{Qwen3-8B} backbone and are fine-tuned on 5{,}000 prompts from \texttt{UltraInteract}~\citep{ultrainteract}, a multi-turn coding dataset with diverse programming problems and solution styles. We initialize from a DPO-tuned reference policy and use a smaller \texttt{Qwen3-1.7B} reward model. We evaluate greedy pass@1 on \texttt{HumanEval}~\citep{humaneval}, \texttt{MBPP}~\citep{mbpp}, and \texttt{LiveCodeBench}~\citep{livecodebench} via the \texttt{OpenCompass} harness~\citep{opencompass}, covering function-level synthesis, broader Python programming patterns, and contamination-resistant competitive problems beyond the training distribution.

\begin{table*}[t]
\centering
\caption{Accuracy of \texttt{Llama3-8B} across five benchmarks. Superscripts report absolute change (percentage points) compared to standard RTO. Best results are in \textbf{bold}.}
\label{tab:five-benchmarks}
\vspace{4pt}
\setlength{\tabcolsep}{6pt}
\renewcommand{\arraystretch}{1.2}
\begin{tabular}{lcccccc}
\toprule
\textbf{Models} & \multicolumn{5}{c}{\textbf{Accuracy}} \\
\cmidrule(lr){2-6}
& \textbf{GSM8K} & \textbf{GSM (CoT)} & \textbf{GSM (ES)} & \textbf{MATH-500} & \textbf{GPQA} \\
\midrule
RTO
& 73.2
& 74.4
& 66.0
& 18.0
& 26.3 \\

SFT
& 61.4\textsuperscript{\textcolor{blue}{-11.8}}
& 64.8\textsuperscript{\textcolor{blue}{-9.6}}
& 63.6\textsuperscript{\textcolor{blue}{-2.4}}
& 17.4\textsuperscript{\textcolor{blue}{-0.6}}
& 24.8\textsuperscript{\textcolor{blue}{-1.5}} \\

PPO
& 73.2\textsuperscript{\textcolor{blue}{0.0}}
& 75.2\textsuperscript{\textcolor{red}{+0.8}}
& 65.6\textsuperscript{\textcolor{blue}{-0.4}}
& 17.0\textsuperscript{\textcolor{blue}{-1.0}}
& 25.8\textsuperscript{\textcolor{blue}{-0.5}} \\

DPO
& 73.9\textsuperscript{\textcolor{red}{+0.7}}
& 76.4\textsuperscript{\textcolor{red}{+2.0}}
& 66.4\textsuperscript{\textcolor{red}{+0.4}}
& 20.1\textsuperscript{\textcolor{red}{+2.1}}
& 20.7\textsuperscript{\textcolor{blue}{-5.6}} \\

FlowRL
& 75.9\textsuperscript{\textcolor{red}{+2.7}}
& 76.2\textsuperscript{\textcolor{red}{+1.8}}
& 66.0\textsuperscript{\textcolor{blue}{0.0}}
& 19.2\textsuperscript{\textcolor{red}{+1.2}}
& \textbf{28.8}\textsuperscript{\textcolor{red}{+2.5}} \\

DAPO
& 78.0\textsuperscript{\textcolor{red}{+4.8}}
& 76.8\textsuperscript{\textcolor{red}{+2.4}}
& 64.0\textsuperscript{\textcolor{blue}{-2.0}}
& 19.4\textsuperscript{\textcolor{red}{+1.4}}
& 25.3\textsuperscript{\textcolor{blue}{-1.0}} \\

\midrule\midrule
$\chi^2$-DRTO
& \textbf{78.7}\textsuperscript{\textcolor{red}{+5.5}}
& \textbf{79.5}\textsuperscript{\textcolor{red}{+5.1}}
& 66.8\textsuperscript{\textcolor{red}{+0.8}}
& \textbf{22.4}\textsuperscript{\textcolor{red}{+4.4}}
& 27.3\textsuperscript{\textcolor{red}{+1.0}} \\

KL-DRTO
& 77.3\textsuperscript{\textcolor{red}{+4.1}}
& 77.4\textsuperscript{\textcolor{red}{+3.0}}
& \textbf{67.6}\textsuperscript{\textcolor{red}{+1.6}}
& 18.8\textsuperscript{\textcolor{red}{+0.8}}
& 26.3\textsuperscript{\textcolor{blue}{0.0}} \\

\bottomrule
\end{tabular}
\vspace{2pt}
\end{table*}

\begin{figure*}[t]
\centering
\begin{minipage}[t]{0.58\linewidth}
\centering
\captionof{table}{Coding accuracy of \texttt{Qwen3-8B} (greedy pass@1). Best results in \textbf{bold}.}
\label{tab:coding-qwen}
\vspace{4pt}
\setlength{\tabcolsep}{6pt}
\renewcommand{\arraystretch}{1.1}
\begin{tabular}{lccc}
\toprule
\textbf{Models} & \textbf{HumanEval} & \textbf{MBPP} & \textbf{LiveCodeBench} \\
\midrule
RTO
& 87.80
& 60.60
& 37.84 \\

SFT
& 84.76\textsuperscript{\textcolor{blue}{-3.04}}
& 69.40\textsuperscript{\textcolor{red}{+8.80}}
& 40.54\textsuperscript{\textcolor{red}{+2.70}} \\

PPO
& 84.15\textsuperscript{\textcolor{blue}{-3.65}}
& 68.60\textsuperscript{\textcolor{red}{+8.00}}
& 37.84\textsuperscript{\textcolor{blue}{0.00}} \\

DPO
& 78.05\textsuperscript{\textcolor{blue}{-9.75}}
& \textbf{70.60}\textsuperscript{\textcolor{red}{+10.00}}
& 24.13\textsuperscript{\textcolor{blue}{-13.71}} \\

FlowRL
& 84.76\textsuperscript{\textcolor{blue}{-3.04}}
& 59.60\textsuperscript{\textcolor{blue}{-1.00}}
& 39.96\textsuperscript{\textcolor{red}{+2.12}} \\

DAPO
& 85.98\textsuperscript{\textcolor{blue}{-1.82}}
& 62.60\textsuperscript{\textcolor{red}{+2.00}}
& \textbf{40.93}\textsuperscript{\textcolor{red}{+3.09}} \\

\midrule\midrule
$\chi^2$-DRTO
& \textbf{88.41}\textsuperscript{\textcolor{red}{+0.61}}
& 69.00\textsuperscript{\textcolor{red}{+8.40}}
& 39.96\textsuperscript{\textcolor{red}{+2.12}} \\

KL-DRTO
& 85.37\textsuperscript{\textcolor{blue}{-2.43}}
& 70.20\textsuperscript{\textcolor{red}{+9.60}}
& 40.54\textsuperscript{\textcolor{red}{+2.70}} \\

\bottomrule
\end{tabular}
\end{minipage}%
\hfill
\begin{minipage}[t]{0.36\linewidth}
\centering
\captionof{table}{End-to-end training cost on the math pipeline.}
\label{tab:runtime}
\vspace{4pt}
\setlength{\tabcolsep}{6pt}
\renewcommand{\arraystretch}{1.1}
\begin{tabular}{lcc}
\toprule
\textbf{Method} & \textbf{Total time} & \textbf{vs.\ RTO} \\
\midrule
RTO              & 10\,h\,00\,m       & $1.00\times$ \\
$\chi^2$-DRTO    & 11\,h\,36\,m       & $1.16\times$ \\
KL-DRTO          & 11\,h\,41\,m       & $1.17\times$ \\
\midrule
DPO              & 2\,h\,07\,m       & $0.21\times$ \\
PPO              & 6\,h\,58\,m       & $0.69\times$ \\
DAPO             & 35\,h\,55\,m       & $3.59\times$ \\
FlowRL           & 48\,h\,38\,m       & $4.86\times$ \\
\bottomrule
\end{tabular}
\end{minipage}
\end{figure*}
We compare DRTO with various baseline models, including (i) RTO~\citep{rto}, the token-level RLHF baseline that our methods extend; (ii) supervised fine-tuning (SFT)~\citep{ouyang2022training} on the same dataset; (iii) DPO~\citep{rafailov2023direct}, a classic preference-based framework; (iv) PPO~\citep{ppo}, a classic policy-gradient RLHF baseline; (v) FlowRL~\citep{flowrl}, a recent flow-matching policy optimization objective; and (vi) DAPO~\citep{dapo}, a recent decoupled-clipping with dynamic-sampling RL objective. Our main methods are reported at the bottom of Table~\ref{tab:five-benchmarks} and ~\ref{tab:coding-qwen}.

\paragraph{DRTO improves robustness across both math and code reasoning.}
Tables~\ref{tab:five-benchmarks} and~\ref{tab:coding-qwen} report results. The five math splits cover both controlled distribution shifts (\texttt{GSM(CoT)} for prompt format, \texttt{GSM(ES)} for language) and broader OOD transfer (\texttt{MATH-500}, \texttt{GPQA}); the three coding benchmarks target function-level synthesis, broader Python patterns, and contamination-resistant competitive problems. Both DRTO methods are the only approaches that do not fall below RTO on any math benchmark, with the largest gains on \texttt{MATH-500} and \texttt{GSM8K}, while every non-robust baseline regresses on at least one math benchmark, most notably DPO on \texttt{GPQA} and DAPO on the language-shifted \texttt{GSM(ES)}. On code reasoning, $\chi^2$-DRTO is the only method that does not regress relative to RTO on any of the three benchmarks; KL-DRTO trades off a drop on \texttt{HumanEval} for sizable gains on \texttt{MBPP} ($+9.6$) and \texttt{LiveCodeBench} ($+2.7$), and is within $0.4$ pp of the best \texttt{LiveCodeBench} score (DAPO at $40.93$). Two patterns in Table~\ref{tab:coding-qwen} deserve a brief comment: DPO falls well below its SFT initialization on both \texttt{HumanEval} ($-6.71$) and \texttt{LiveCodeBench} ($-16.41$), consistent with known offline over-optimization of preference-based objectives on coding data~\citep{chipo}, and SFT itself outperforms RTO on \texttt{MBPP} and \texttt{LiveCodeBench}, indicating that the RTO baseline on this code pipeline is not uniformly strong even before robustification, which makes DRTO's consistent recovery of these splits a non-trivial improvement rather than a result of a weak reference. Each configuration is reported from a single training seed; we therefore frame these numbers as evidence of consistent improvement over RTO rather than as state-of-the-art on every benchmark.

\paragraph{Runtime comparison.}
Table~\ref{tab:runtime} reports end-to-end wall-clock training time on identical hardware. Both DRTO variants run at essentially the same cost as standard RTO, whereas FlowRL and DAPO are several times more expensive. DRTO therefore delivers its robustness gains without meaningful additional training overhead. A representative qualitative example, in which RTO drops a step that both DRTO variants recover, is provided in Appendix~\ref{app:qualitative_examples}.

\subsection{Performance Analysis}
\label{subsec:ablations}
In order to better understand the outperformance of DRTO, we conduct the following ablation studies. More experiments can be found in Appendix~\ref{app:training_dynamics}.

\paragraph{The vital role of token-wise reward: DRTO vs DRPPO.}
Despite the large gains from DRTO, a natural question is whether the span-level robust objective alone is doing all the work, or whether the token-wise reward shaping inherited from RTO still plays an essential role. To answer this, we conduct the following ablation. Since our robust objectives act purely at the span level through the minibatch losses $\{L_i(\theta)\}$, both the KL and $\chi^2$ variants can be plugged into vanilla PPO with minimal changes: we simply replace the trajectory-level reward in PPO with the same span-level robust reweighting used by DRTO, yielding KL-DRPPO and $\chi^2$-DRPPO. Both models are trained under the same configuration as Section~\ref{subsec:main_exp}, and Table~\ref{tab:drto_ablation} compares them against the corresponding KL-DRTO and $\chi^2$-DRTO models.

Across all three benchmarks, DRTO is consistently stronger than its DRPPO counterpart, with the gap widening on the harder \texttt{MATH-500} split where the robust reweighting matters most. This indicates that the gains do not come from span-level DRO alone. Rather, the token-wise reward in RTO supplies a dense, fine-grained learning signal that PPO's sparse trajectory-level reward cannot replicate, while the span-level DRO term stabilizes the update by upweighting loss-sensitive spans within each minibatch so that high-loss segments receive proportionally more gradient signal. When the token-wise shaping is removed in DRPPO, the improvement from robustness alone is markedly smaller and less consistent under shift; when both ingredients are combined in DRTO, robustness and accuracy improve together. In short, span-level DRO improves the reliability of the update, but the intrinsic token-wise reward in RTO still supplies an indispensable part of the learning signal, and DRTO benefits from both.

\paragraph{Span-level vs.\ trajectory-level DRTO.}
A second natural question is whether the span-level segmentation in DRTO is essential, or whether the same robustness could be obtained by treating each full response as a single span and applying the DRO reweighting at the trajectory level. To isolate the contribution of span segmentation, we train trajectory-level variants of both robust objectives, denoted KL-DRTO (w/o span) and $\chi^2$-DRTO (w/o span), under the same setup as Section~\ref{subsec:main_exp}. These variants retain the token-wise RTO reward and the same KL or $\chi^2$ ambiguity set, but compute the inner maximization over whole trajectories rather than over the variable-length spans produced by our segmentation rule.

Table~\ref{tab:drto_ablation} shows that span-level DRTO is uniformly stronger than its trajectory-level counterpart on the harder, more shift-sensitive benchmarks (\texttt{MATH-500} and \texttt{GSM8K}), while remaining comparable on \texttt{GPQA}. The reason is structural: long chain-of-thought responses contain a mixture of high-confidence routine steps and a few low-confidence pivot tokens at which the model commits to a reasoning branch. A trajectory-level DRO weight averages across both kinds of tokens and therefore cannot localize the few segments that genuinely drive the loss, so the worst-case inner maximization is diluted. Span-level cuts at low-confidence positions instead concentrate the robust reweighting on the segments where errors actually accumulate, letting the DRO term penalize precisely the parts of the response that are most sensitive to distribution shift. Span segmentation is therefore not a cosmetic detail but the mechanism that makes the robust reweighting actionable on long, structurally heterogeneous reasoning trajectories.

\begin{table*}[t]
\centering
\scriptsize
\setlength{\tabcolsep}{3pt}
\renewcommand{\arraystretch}{0.95}

\begin{minipage}[t]{0.65\linewidth}
\vspace{0pt}\centering
\caption{Comparison of DRPPO, trajectory-level DRTO (no span), and span-level DRTO. Best results in each column are in \textbf{bold}.}
\label{tab:drto_ablation}
\resizebox{\linewidth}{!}{%
\begin{tabular}[t]{lccc}
\toprule
\textbf{Model} & \textbf{GSM8K} & \textbf{MATH-500} & \textbf{GPQA} \\
\midrule
$\chi^2$-DRPPO & 77.4 & 14.2 & 26.8  \\
KL-DRPPO       & 77.0 & 15.8 & 25.8  \\
\midrule
$\chi^2$-DRTO (w/o span)  & 77.7 & 18.6 & 25.8 \\
KL-DRTO (w/o span)        & 76.8 & 18.6 & 26.3 \\
\midrule
$\chi^2$-DRTO    & \textbf{78.7} & \textbf{22.4} & \textbf{27.3} \\
KL-DRTO          & 77.3 & 18.8 & 26.3 \\
\bottomrule
\end{tabular}}
\end{minipage}
\hfill
\begin{minipage}[t]{0.32\linewidth}
\vspace{0pt}\centering
\caption{In-distribution evaluation on GSM8K.}
\label{tab:in_dist}
\resizebox{\linewidth}{!}{%
\begin{tabular}[t]{lc}
\toprule
\textbf{Method} & \textbf{GSM8K}\\
\midrule
SFT & 63.3  \\
PPO & 78.1  \\
DPO & 76.8  \\
RTO & \textbf{80.2}   \\
\textsc{KL-DRTO} & 80.1  \\
\textsc{$\chi^2$-DRTO} & 79.8 \\
\bottomrule
\end{tabular}}
\end{minipage}

\end{table*}

\paragraph{In-distribution case.}
To verify that robust objectives do not distort behavior on the training distribution, we additionally train on the \texttt{GSM8K} training split and evaluate on its test split. Table~\ref{tab:in_dist} shows that both KL-DRTO and $\chi^2$-DRTO closely track RTO and remain ahead of PPO and DPO, confirming that DRTO's robustness gains do not come at the expense of in-distribution accuracy.

\section{Conclusion}
\label{sec:conclusion}

We introduced \emph{Distributionally Robust Token Optimization} (DRTO), a framework that incorporates distributionally robust optimization into token-level RLHF by constructing divergence-based ambiguity sets around the empirical span distribution and deriving tight dual surrogates, yielding more stable training dynamics. Our analysis further shows that $\chi^2$-DRTO and KL-DRTO are two concrete instantiations of the same robust optimization principle, offering a clean theoretical interpretation of adaptive-penalty mechanisms commonly used in practice. As a drop-in replacement in standard RLHF pipelines, DRTO improves over standard RTO on most out-of-distribution math and code evaluations with no increase in model size or architectural change. Limitations include evaluation on only two 8B backbones (\texttt{Llama3-8B}, \texttt{Qwen3-8B}) and single-turn settings, the lack of formal guarantees against adversarial or worst-case safety failures, and the rollout cost shared with other on-policy RLHF methods. Meanwhile, promising directions consist of extensions to multi-turn trajectories, larger models, and integration with preference-based or multi-reward RLHF.

\bibliography{ref}
\bibliographystyle{plainnat}

\newpage
\appendix

\section{Appendix of Proofs}

\subsection{Proof of Theorem~\ref{thm:kl_drto_main}}
\label{app:kl_drto_proofs}

We prove the exact one-dimensional dual representation of the span-level KL-robust value and the induced exponential-tilt form of the worst-case span reweighting when the dual infimum is attained at a finite temperature.

Let $m=(i,s)$ index a span in $\mathcal S_{\mathcal B}$, write $z_m=L_{i,s}(\theta)$, and let $\mathbb{P}_{\mathcal S}(m)$ be the response-balanced nominal probability. If $\rho=0$, then the constraint $D_{\mathrm{KL}}(\mathbb{Q}\Vert \mathbb{P}_{\mathcal S})\le 0$ forces $\mathbb{Q}=\mathbb{P}_{\mathcal S}$, so $R_{\mathrm{KL}}(\theta;0)=\sum_m \mathbb{P}_{\mathcal S}(m)\, z_m$. Moreover, $\Psi_\eta(\theta)\to \sum_m \mathbb{P}_{\mathcal S}(m)\,z_m$ as $\eta\downarrow 0$, hence \eqref{eq:kl_exact_dual_main} holds when $\rho=0$. We now assume $\rho>0$.

The inner maximization in \eqref{eq:generic_span_dro} can be written as
\[
\max_{\mathbb{Q}\in\Delta(\mathcal S_{\mathcal B})}\ \sum_m \mathbb{Q}(m) z_m
\quad\text{where}\quad
D_{\mathrm{KL}}(\mathbb{Q}\Vert \mathbb{P}_{\mathcal S})\le \rho.
\]
The point $\mathbb{Q}=\mathbb{P}_{\mathcal S}$ is strictly feasible whenever $\rho>0$, so Slater's condition~\cite{BoydVandenberghe04} holds and strong duality applies. Let $\nu\ge0$ be the Lagrange multiplier for the KL constraint. Then
\begin{align*}
R_{\mathrm{KL}}(\theta;\rho)
&=
\inf_{\nu\ge0}\left\{
\nu\rho+
\sup_{\mathbb{Q}\in\Delta(\mathcal S_{\mathcal B})}
\left(\sum_m \mathbb{Q}(m) z_m-\nu D_{\mathrm{KL}}(\mathbb{Q}\Vert \mathbb{P}_{\mathcal S})\right)
\right\}.
\end{align*}
For $\nu>0$, set $\eta=1/\nu$. By the Gibbs variational principle~\citep{vanHandel16},
\[
\sup_{\mathbb{Q}\in\Delta(\mathcal S_{\mathcal B})}
\left(\sum_m \mathbb{Q}(m) z_m-\frac{1}{\eta}D_{\mathrm{KL}}(\mathbb{Q}\Vert \mathbb{P}_{\mathcal S})\right)
=
\frac{1}{\eta}\log\sum_m \mathbb{P}_{\mathcal S}(m)\, e^{\eta z_m}
=
\Psi_\eta(\theta),
\]
and the maximizer is
\[
\mathbb{Q}_\eta(m)=\frac{\mathbb{P}_{\mathcal S}(m)\, e^{\eta z_m}}{\sum_{m'}\mathbb{P}_{\mathcal S}(m')\, e^{\eta z_{m'}}},
\]
which is exactly \eqref{eq:softmax_Q_main}. Substituting $\nu=1/\eta$ gives \eqref{eq:kl_exact_dual_main}; the case $\nu=0$ is recovered as the $\eta\to\infty$ limit. If the infimum is attained at a finite $\eta^\star$, strong duality and the KKT conditions imply that the corresponding primal optimizer is $\mathbb{Q}_{\eta^\star}$, which gives \eqref{eq:kl_primal_value_main}.

\subsection{Proof of Theorem~\ref{thm:chi2_drto_main}}
\label{app:chi2_drto_proofs}

We first show the following bound, which also characterizes an optimizer when the nonnegativity condition holds.

\begin{theorem}[$\chi^2$ span-robust bound]
\label{thm:chi2_mean_std_app}
Fix a minibatch $\mathcal B$ and its span set $\mathcal S_{\mathcal B}$. Let
\[
\bar L(\theta)=\E_{(i,s)\sim\mathbb{P}_{\mathcal S}}[L_{i,s}(\theta)],
\qquad
\sigma(\theta)=\sqrt{\Var_{(i,s)\sim\mathbb{P}_{\mathcal S}}(L_{i,s}(\theta))}.
\]
Then for any $\rho>0$,
\begin{equation}
R_{\chi^2}(\theta;\rho)
\le
\bar L(\theta)+\sigma(\theta)\sqrt{\rho}.
\label{eq:chi2_upper_bound_app}
\end{equation}
Moreover, given $z_{i,s}(\theta)=L_{i,s}(\theta)-\bar L(\theta)$, if
\begin{equation}
1+\sqrt{\rho}\,\frac{z_{i,s}(\theta)}{\sigma(\theta)}\ge 0\quad \text{for all }(i,s),
\label{eq:chi2_nonneg_condition_app}
\end{equation}
then the bound in \eqref{eq:chi2_upper_bound_app} is tight and the supremum in \eqref{eq:chi2_robust_value} equals $\bar L(\theta)+\sigma(\theta)\sqrt{\rho}$. In this case, an optimal adversary is
\begin{equation}
\mathbb{Q}^\star(i,s)
=
\mathbb{P}_{\mathcal S}(i,s)\left(1+\sqrt{\rho}\,\frac{z_{i,s}(\theta)}{\sigma(\theta)}\right).
\label{eq:chi2_opt_weights_app}
\end{equation}
\end{theorem}

\begin{proof}
Let $w_{i,s}=\mathbb{Q}(i,s)/\mathbb{P}_{\mathcal S}(i,s)$. Then $\E_{\mathbb{P}_{\mathcal S}}[w]=1$ and $\E_{\mathbb{P}_{\mathcal S}}[(w-1)^2]\le\rho$. Adding and subtracting $\bar L(\theta)$ inside the inner expectation and using $\E_{\mathbb{P}_{\mathcal S}}[w]=1$ gives the mean-shift identity
\[
\sum_{(i,s)}\mathbb{Q}(i,s)\, L_{i,s}(\theta)
=
\E_{\mathbb{P}_{\mathcal S}}[w(L_{i,s}(\theta)-\bar L(\theta))]+\bar L(\theta)
=
\bar L(\theta)+\E_{\mathbb{P}_{\mathcal S}}\big[(w-1)(L_{i,s}(\theta)-\bar L(\theta))\big],
\]
where the second equality uses the fact that $\E_{\mathbb{P}_{\mathcal S}}[L_{i,s}(\theta)-\bar L(\theta)]=0$, so the contribution of the centered random variable under the constant weight $1$ vanishes.
By Cauchy--Schwarz,
\[
\E_{\mathbb{P}_{\mathcal S}}\big[(w-1)(L_{i,s}(\theta)-\bar L(\theta))\big]
\le
\sqrt{\E_{\mathbb{P}_{\mathcal S}}[(w-1)^2]}
\sqrt{\E_{\mathbb{P}_{\mathcal S}}[(L_{i,s}(\theta)-\bar L(\theta))^2]}
\le
\sqrt{\rho}\,\sigma(\theta),
\]
which proves \eqref{eq:chi2_upper_bound_app}.

If $\sigma(\theta)=0$, then all span losses are equal under $\mathbb{P}_{\mathcal S}$, and the theorem is immediate. Now assume $\sigma(\theta)>0$ and \eqref{eq:chi2_nonneg_condition_app} holds. Define
$w^\star_{i,s}=1+\sqrt{\rho}\,z_{i,s}(\theta)/\sigma(\theta)$ and $\mathbb{Q}^\star(i,s)=\mathbb{P}_{\mathcal S}(i,s)\, w^\star_{i,s}$. Then $\mathbb{Q}^\star$ is a valid distribution, $\E_{\mathbb{P}_{\mathcal S}}[w^\star]=1$, and
\[
\E_{\mathbb{P}_{\mathcal S}}[(w^\star-1)^2]
=
\E_{\mathbb{P}_{\mathcal S}}\!\left[\rho\frac{z_{i,s}(\theta)^2}{\sigma(\theta)^2}\right]
=
\rho,
\]
so $\mathbb{Q}^\star\in\Omega^{\chi^2}_\rho(\mathbb{P}_{\mathcal S})$. Finally,
\[
\sum_{(i,s)}\mathbb{Q}^\star(i,s)\, L_{i,s}(\theta)
=
\bar L(\theta)+\sqrt{\rho}\frac{\E_{\mathbb{P}_{\mathcal S}}[z_{i,s}(\theta)^2]}{\sigma(\theta)}
=
\bar L(\theta)+\sigma(\theta)\sqrt{\rho},
\]
which matches the upper bound and proves tightness.
\end{proof}

Theorem~\ref{thm:chi2_drto_main} follows by applying Theorem~\ref{thm:chi2_mean_std_app} to $R_{\chi^2}(\theta;\rho)$.

\section{Training Dynamics}
\label{app:training_dynamics}

This appendix complements the main results with optimization-side diagnostics. Figure~\ref{fig:training_dynamics} reports training-time return, critic loss, and PTX loss across methods, providing further evidence that DRTO's robustness gains do not come at the cost of less stable optimization.

\begin{figure*}[t]
    \centering
    \begin{subfigure}[t]{0.32\linewidth}
        \centering
        \includegraphics[width=\linewidth]{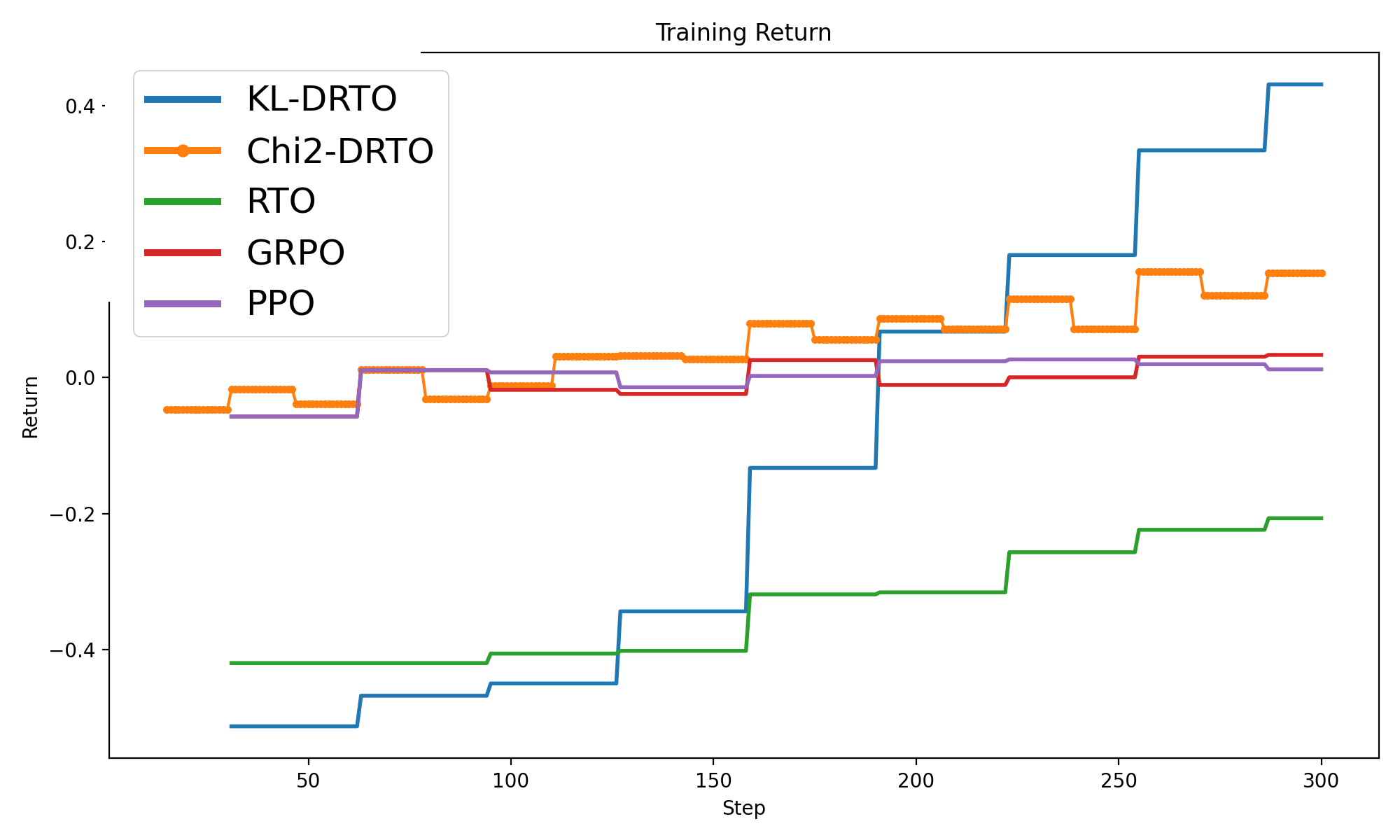}
        \caption{Training return comparison.}
        \label{fig:train_return}
    \end{subfigure}
    \hfill
    \begin{subfigure}[t]{0.32\linewidth}
        \centering
        \includegraphics[width=\linewidth]{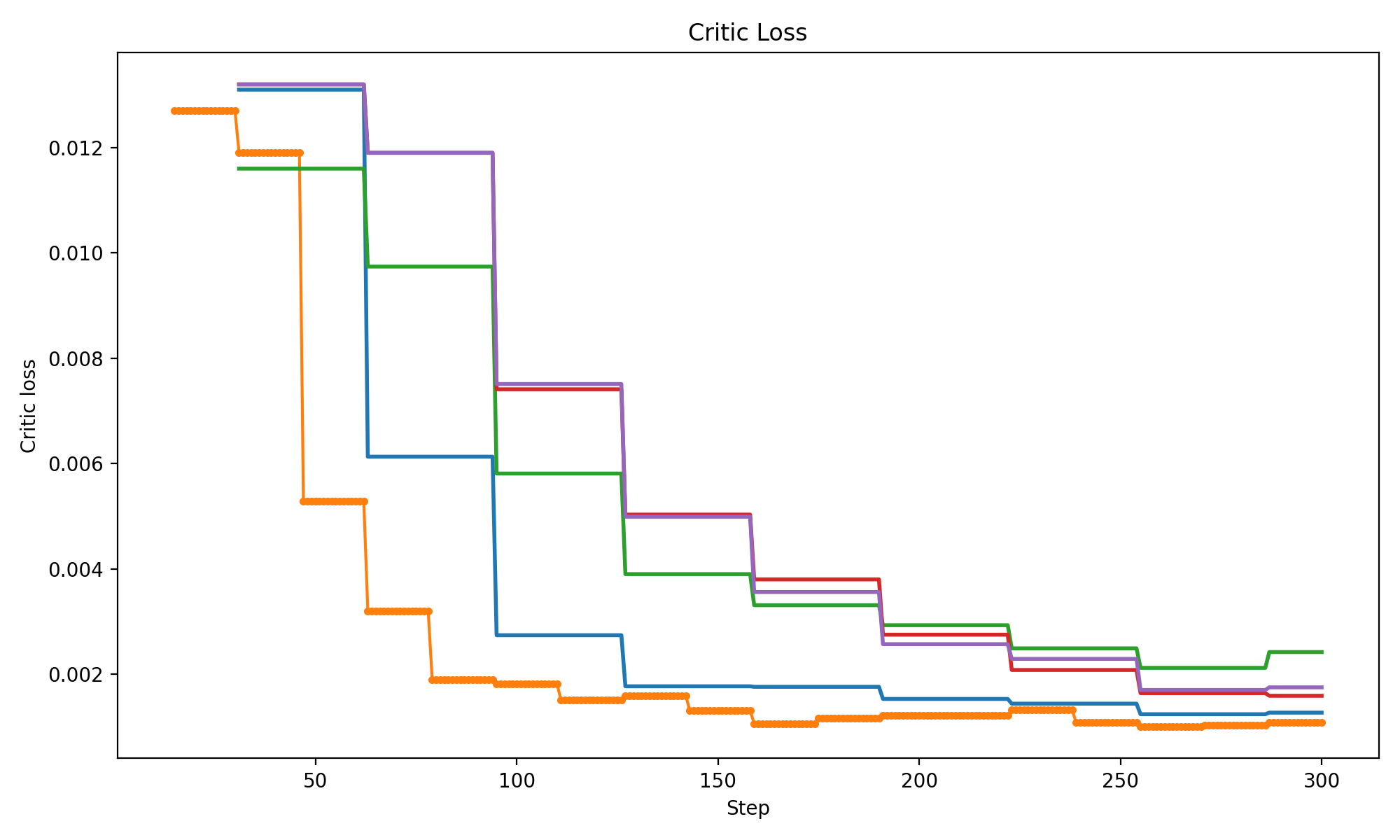}
        \caption{Critic loss during training.}
        \label{fig:critic_loss}
    \end{subfigure}
    \hfill
    \begin{subfigure}[t]{0.32\linewidth}
        \centering
        \includegraphics[width=\linewidth]{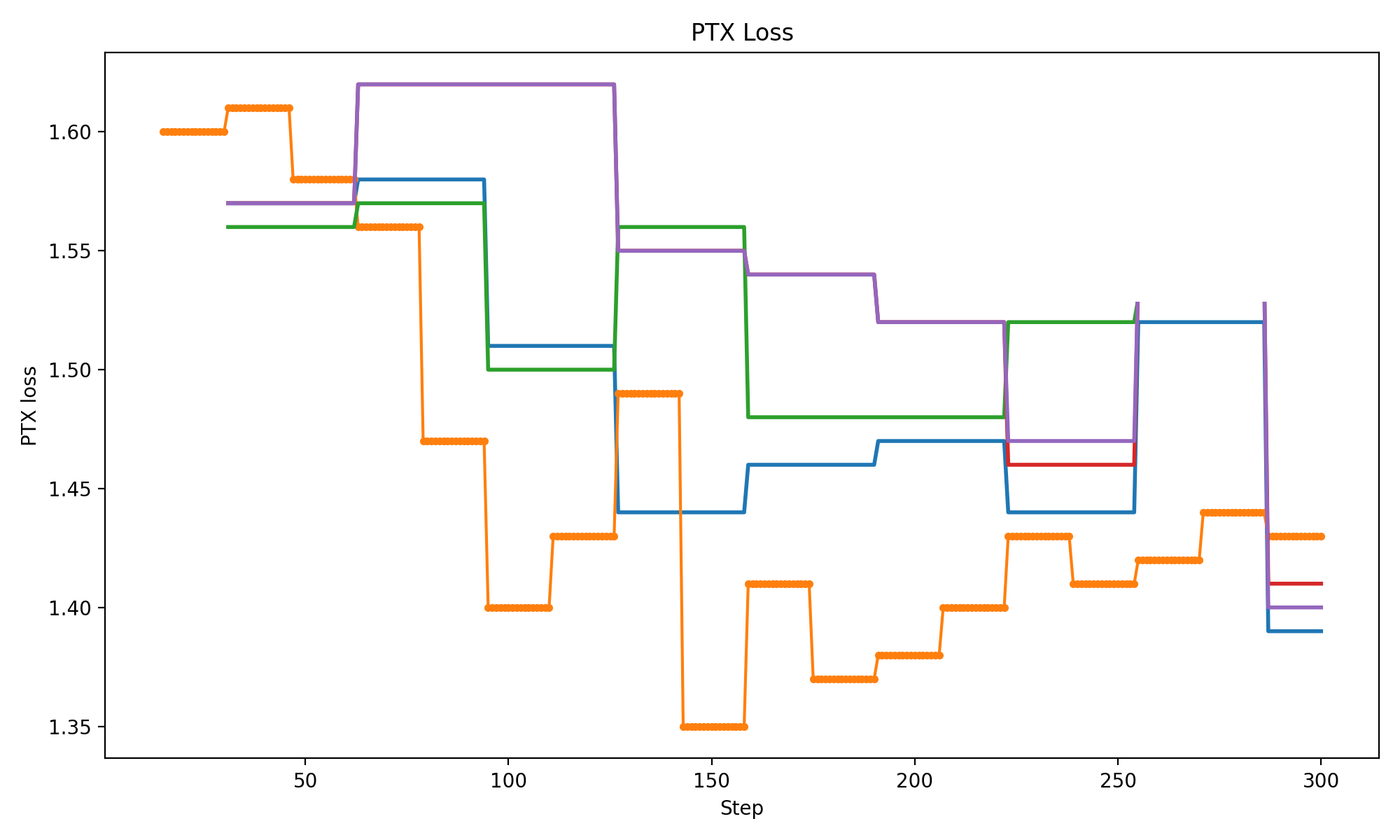}
        \caption{PTX loss comparison.}
        \label{fig:ptx_loss}
    \end{subfigure}
    \caption{Training dynamics comparison across methods.}
    \label{fig:training_dynamics}
\end{figure*}

\paragraph{Convergence and optimization effectiveness.}
Figure~\ref{fig:train_return} compares training return across methods. KL-DRTO attains the highest return with a clear upward trend, while $\chi^2$-DRTO consistently outperforms PPO and RTO; PPO stays near zero and RTO remains negative for most of training. To quantify this, we measure the number of optimization steps needed to first reach $80\%$ of each method's own peak training return: KL-DRTO and $\chi^2$-DRTO cross this threshold within roughly the first third of training, whereas PPO and RTO require more than twice as many steps and, in the case of RTO, never reach a comparable absolute return level. This suggests DRTO improves not only final performance but also the optimization path itself.

\paragraph{Critic stability.}
Figure~\ref{fig:critic_loss} reports the critic loss. Both DRTO variants converge to a noticeably lower and more stable critic loss than PPO and RTO, with $\chi^2$-DRTO the lowest, indicating better-conditioned value learning that supports stable policy optimization.

\paragraph{Preserving the pretraining objective.}
Figure~\ref{fig:ptx_loss} reports the auxiliary PTX loss, which tracks deviation from the pretraining distribution. Both DRTO variants maintain a lower PTX loss than PPO and RTO; in particular, KL-DRTO achieves strong return improvements without drifting away from the base language model.

\section{Implementation Details}
\label{app:implementation}

This appendix provides implementation details for all baselines and robust variants in our DRTO study,
including the training pipeline, shared hyperparameters, and method-specific configurations.

\paragraph{Training pipeline.}
All experiments use \texttt{OpenRLHF/Llama-3-8b-sft-mixture} as the SFT initialization for the actor.
For methods with PPO updates, we use a shared reward model checkpoint as the reward function and as the critic initialization.
For RTO-style methods, we additionally load a DPO policy checkpoint to provide token-level preference shaping.

\paragraph{Datasets.}
Preference learning (DPO) uses the UltraFeedback preference dataset.
Reinforcement learning (PPO/RTO/DRTO) uses a prompt-only OpenMath dataset.
Unless stated otherwise, we use \texttt{max\_samples=10000} prompts and \texttt{seed=42}.

\paragraph{Shared RLHF setup.}
Across PPO/RTO/DRTO runs, We apply a cosine learning rate schedule with 3\%
warming steps and 10\% minimum learning rate.
We use DeepSpeed ZeRO-2 with \texttt{adam\_offload}, \texttt{bf16},
\texttt{flash\_attn}, and \texttt{gradient\_checkpointing}.

\paragraph{PPO implementation details.}
Our PPO-based runs follow the OpenRLHF implementation and enable reward normalization. We use GAE with \texttt{lambd=0.95} and \texttt{gamma=1.0}. For DRTO, span cut points are selected from stored rollout probabilities with minimum span length 16, maximum span number 8, and cut probability threshold $p_{\mathrm{cut}}=0.05$; each span loss is the mean token-level PPO loss in that span, and the nominal span distribution is response-balanced as in \eqref{eq:span_nominal_distribution}. The DRO statistics ($\bar L$, $\sigma$ for $\chi^2$-DRTO; $\eta^\star$ and $\mathbb{Q}^\star$ for KL-DRTO) are computed once per train batch (256 prompt-response rollouts) over all spans in that batch, so the adversary sees the full minibatch span population rather than a per-microbatch slice. The adversarial weights $\mathbb{Q}^\star$ are detached (no gradient through the dual variable), so the resulting actor loss is a weighted sum of differentiable per-span PPO losses.
We list the hyperparameters for each method below.
\begin{center}
\begin{minipage}[t]{0.48\linewidth}
\vspace{0pt}\centering

\begin{tabular}[t]{p{0.62\linewidth} p{0.30\linewidth}}
\hline
\multicolumn{2}{c}{\textbf{DPO (UltraFeedback)}}\\
\hline
Learning rate & $5\text{e-}7$ \\
Batch size & 256 \\
DPO KL coefficient ($\beta$) & 0.1 \\
\hline
\end{tabular}

\vspace{0.8em}

\begin{tabular}[t]{p{0.62\linewidth} p{0.30\linewidth}}
\hline
\multicolumn{2}{c}{\textbf{RTO (OpenMath)}}\\
\hline
Rollout / train batch size & 256 / 256 \\
Actor / critic learning rate & 1e-5 / 5e-5 \\
Initial KL coefficient & 0.01 \\
DPO reward scale / clip & 0.05 / 0.05 \\
PTX coefficient & 0.05 \\
\hline
\end{tabular}

\end{minipage}
\hfill
\begin{minipage}[t]{0.48\linewidth}
\vspace{0pt}\centering

\begin{tabular}[t]{p{0.62\linewidth} p{0.30\linewidth}}
\hline
\multicolumn{2}{c}{\textbf{PPO (OpenMath)}}\\
\hline
Rollout / train batch size & 512 / 256 \\
Actor / critic learning rate & 1e-5 / 5e-5 \\
Initial KL coefficient & 0.01 \\
PPO clip ($\epsilon$) & 0.2 \\
PTX coefficient & 0.05 \\
\hline
\end{tabular}

\end{minipage}
\end{center}

\vspace{0.6em}

\begin{center}
\begin{minipage}[t]{0.48\linewidth}
\vspace{0pt}\centering

\begin{tabular}[t]{p{0.62\linewidth} p{0.30\linewidth}}
\hline
\multicolumn{2}{c}{\textbf{KL-DRTO (OpenMath)}}\\
\hline
Rollout / train batch size & 512 / 256 \\
Actor / critic learning rate & 1e-5 / 5e-5 \\
Initial KL coefficient & 0.005 \\
DPO reward scale / clip & 0.05 / 0.05 \\
PTX coefficient & 0.05 \\
Span min len / max num & 16 / 8 \\
Cut probability $p_{\mathrm{cut}}$ & 0.05 \\
\hline
KL radius $\rho$ & 0.01 \\
\hline
\end{tabular}

\end{minipage}
\hfill
\begin{minipage}[t]{0.48\linewidth}
\vspace{0pt}\centering

\begin{tabular}[t]{p{0.62\linewidth} p{0.30\linewidth}}
\hline
\multicolumn{2}{c}{\textbf{$\chi^2$-DRTO (OpenMath)}}\\
\hline
Rollout / train batch size & 512 / 256 \\
Actor / critic learning rate & 1e-5 / 5e-5 \\
Initial KL coefficient & 0.005 \\
DPO reward scale / clip & 0.05 / 0.05 \\
PTX coefficient & 0.05 \\
Span min len / max num & 16 / 8 \\
Cut probability $p_{\mathrm{cut}}$ & 0.05 \\
\hline
$\chi^2$ radius $\rho$ & 0.1 \\
\hline
\end{tabular}

\end{minipage}
\end{center}

\paragraph{Code pipeline.}
For coding experiments we replace the math pipeline with a code-aligned counterpart. The actor is initialized from \texttt{Qwen3-8B}; the reward model is a \texttt{Qwen3-1.7B} preference model, which also serves as the critic initialization. Preference learning (DPO) uses the \texttt{UltraInteract}~\citep{ultrainteract} preference dataset, and the reinforcement-learning stage uses prompt-only rollouts drawn from \texttt{UltraInteract}. Evaluation uses \texttt{HumanEval}~\citep{humaneval}, \texttt{MBPP}~\citep{mbpp,mbppplus}, and \texttt{LiveCodeBench}~\citep{livecodebench} via the \texttt{OpenCompass}~\citep{opencompass} harness with greedy decoding (pass@1). All other training hyperparameters (batch sizes, optimizer, KL/PTX coefficients, GAE parameters, span-construction settings, and ambiguity radii) follow the corresponding rows of the math pipeline above.

\section{Response comparison}
\label{app:qualitative_examples}
To further illustrate qualitative differences among methods, we present a representative evaluation example and compare the responses produced by RTO and both DRTO variants on the same prompt in Table~\ref{tab:qualitative_drto_example}.

\begin{table*}[t]
\caption{Example responses on the same prompt. RTO misses the pack price multiplication, while both DRTO variants return the correct total cost.}
\centering
\small
\setlength{\tabcolsep}{6pt}
\begin{tabular}{p{0.97\textwidth}}
\toprule
\textbf{Question.} Terry eats 2 yogurts a day. They are currently on sale at 4 yogurts for \$5.00. How much does he spend on yogurt over 30 days? \\
\midrule
\textbf{RTO.}\newline
{\footnotesize\ttfamily
Terry eats 2 yogurts a day, so over 30 days he eats 2*30=<<2*30=60>>60 yogurts.\newline
They are currently on sale at 4 yogurts for \$5.00, so Terry spends 60/4=\$<<60/4=15>>15.00 on yogurt over 30 days.\newline
\#\#\#\# 15
} \\
\midrule
\textbf{$\chi^2$-DRTO ($\rho=0.1$).}\newline
{\footnotesize\ttfamily
Terry eats 2 yogurts a day, so over 30 days he eats 2*30=<<2*30=60>>60 yogurts.\newline
They are currently on sale at 4 yogurts for \$5.00, so Terry spends \$5.00 for every 4 yogurts.\newline
To find out how much Terry spends on 60 yogurts, we can divide 60 by 4 to find out how many sets of 4 yogurts he needs: 60/4=<<60/4=15>>15 sets of 4 yogurts.\newline
Since each set of 4 yogurts costs \$5.00, Terry spends 15*\$5.00=<<15*5=75>>\$75.00 on yogurt over 30 days.\newline
\#\#\#\# 75
} \\
\midrule
\textbf{KL-DRTO ($\rho=0.01$).}\newline
{\footnotesize\ttfamily
Terry eats 2 yogurts a day, so over 30 days he eats 2*30=<<2*30=60>>60 yogurts.\newline
The sale is 4 yogurts for \$5.00, so the price per yogurt is \$5.00/4 = \$<<5/4=1.25>>1.25.\newline
Therefore, Terry spends 60 * \$1.25 = \$<<60*1.25=75>>75 on yogurt over 30 days.\newline
\#\#\#\# 75
} \\
\bottomrule
\end{tabular}
\vspace{2pt}
\label{tab:qualitative_drto_example}
\end{table*}

\newpage

\end{document}

%% file: math_commands.tex
\usepackage{amsmath,amsfonts,bm}

\def\eqref#1{equation~\ref{#1}}

\def\1{\bm{1}}

\DeclareMathAlphabet{\mathsfit}{\encodingdefault}{\sfdefault}{m}{sl}
\SetMathAlphabet{\mathsfit}{bold}{\encodingdefault}{\sfdefault}{bx}{n}

\newcommand{\E}{\mathbb{E}}

\newcommand{\Var}{\mathrm{Var}}



%% file: lpsymbols.tex
\usepackage{amsmath,amssymb}

\def\bm{{\mathbf m}}

\def\texitem#1{\par\smallskip\noindent\hangindent 25pt
               \hbox to 25pt {\hss #1 ~}\ignorespaces}

